%% file: supcon_cpd.tex
\documentclass[conference]{IEEEtran}
\IEEEoverridecommandlockouts
\input{math_commands.tex}

\usepackage{cite}
\usepackage{amsmath,amssymb,amsfonts,amsthm}
\usepackage{algorithmic}
\usepackage{graphicx}
\usepackage{textcomp}
\usepackage{xcolor}
\usepackage{multirow}
\newtheorem{proposition}{Proposition}
\usepackage{stackengine}
\newcommand\xrowht[2][0]{\addstackgap[.5\dimexpr#2\relax]{\vphantom{#1}}}
\def\BibTeX{{\rm B\kern-.05em{\sc i\kern-.025em b}\kern-.08em
    T\kern-.1667em\lower.7ex\hbox{E}\kern-.125emX}}
\def\BibTeX{{\rm B\kern-.05em{\sc i\kern-.025em b}\kern-.08em
    T\kern-.1667em\lower.7ex\hbox{E}\kern-.125emX}}
\makeatletter
\newcommand{\linebreakand}{%
  \end{@IEEEauthorhalign}
  \hfill\mbox{}\par
  \mbox{}\hfill\begin{@IEEEauthorhalign}
}
\makeatother
\begin{document}
  \title{Time Series Representation Learning with Supervised Contrastive Temporal Transformer}
  \author{
   \IEEEauthorblockN{Yuansan Liu\IEEEauthorrefmark{1}, Sudanthi Wijewickrema\IEEEauthorrefmark{2}, Christofer Bester\IEEEauthorrefmark{2}, Stephen J. O'Leary\IEEEauthorrefmark{2}, James Bailey\IEEEauthorrefmark{1}}
  \IEEEauthorblockA{\IEEEauthorrefmark{1}\textit{School of Computing and Information Systems, The University of Melbourne, Melbourne, Australia} \\
  yuansanl@student.unimelb.edu.au, baileyj@unimelb.edu.au}

  \IEEEauthorblockA{\IEEEauthorrefmark{2}\textit{Department of Surgery (Otolaryngology) The University of Melbourne, Melbourne, Australia}\\
  \{sudanthi.wijewickrema,christofer.bester,sjoleary\}@unimelb.edu.au}
  }
  \maketitle

  \begin{abstract}
   Finding effective representations for time series data is a useful but challenging task. Several works utilize self-supervised or unsupervised learning methods to address this. However, there still remains the open question of how to leverage available label information for better representations. To answer this question, we exploit pre-existing techniques in time series and representation learning domains and develop a simple, yet novel fusion model, called: \textbf{S}upervised \textbf{CO}ntrastive \textbf{T}emporal \textbf{T}ransformer (SCOTT). We first investigate suitable augmentation methods for various types of time series data to assist with learning change-invariant representations. Secondly, we combine Transformer and Temporal Convolutional Networks in a simple way to efficiently learn both global and local features. Finally, we simplify Supervised Contrastive Loss for representation learning of labelled time series data. We preliminarily evaluate SCOTT on a downstream task, Time Series Classification, using 45 datasets from the UCR archive. The results show that with the representations learnt by SCOTT, even a weak classifier can perform similar to or better than existing state-of-the-art models (best performance on 23/45 datasets and highest rank against 9 baseline models). Afterwards, we investigate SCOTT's ability to address a real-world task, online Change Point Detection (CPD), on two datasets: a human activity dataset and a surgical patient dataset. We show that the model performs with high reliability and efficiency on the online CPD problem ($\sim$98\% and $\sim$97\% area under precision-recall curve respectively). Furthermore, we demonstrate the model's potential in tackling early detection and show it performs best compared to other candidates.\\
  \end{abstract}


  \section{Introduction}

   A good representation encodes useful information from the original data into a vector, and makes it possible to analyze complex data effectively. Contrastive Learning, a well known branch of representation learning, has seen high activity in recent years due to its success in computer vision and natural language processing tasks. It extracts analogous features from similar objects and disparate features from non-similar ones, and consequently learns distinctive representations for downstream tasks such as classification. Several existing works employ self-supervised or unsupervised contrastive learning frameworks to learn time series representations and achieve state-of-the-art (SOTA) results \cite{eldele2021tsrl,yue2022ts2vec,yang2022timeclr}. Nonetheless, several challenges remain to be addressed. First, time series data have diverse temporal properties, making it critical to design proper augmentation strategies for different types of data \cite{iwana2021tsaugsuv}. Second, these data contains both global and local properties, necessitating the learning of both attributes efficiently and effectively. Third, despite limited access to information such as labels, it's crucial to leverage the available ones for better representation learning.


   In this paper, we build on established methods to address the aforementioned challenges in time series representation learning. We first investigate suitable augmentation methods for various types of time series data to support contrastive learning \cite{chen2020simple,khosla2020supervised}. In addition, we introduce a unique augmentation strategy for the online Change Point Detection (CPD) task. Secondly, we merge Transformer and Temporal Convolutional Networks by replacing the former's feed forward convolutions with the latter's dilated causal convolutions. With such simple fusion, the model learns both global and local features from time series data. Thirdly, we adapt and simplify Supervised Contrastive (SupCon) Loss \cite{khosla2020supervised} to leverage label information and make it work efficiently with time series data. Combining all these components, we introduce the \textbf{S}upervised \textbf{CO}ntrastive \textbf{T}emporal \textbf{T}ransformer (SCOTT). We do an initial evaluation via a downstream task, Time Series Classification (TSC), using 45 datasets from the UCR Time Series Archive \footnote{https://www.cs.ucr.edu/~eamonn/time\_series\_data\_2018/} \cite{UCRArchive2018}. The results show that with the  representations learnt by our model, a Multi Layer Perceptron which performs worst against baseline models, can now attain the highest rank, highlighting the effectiveness of the learnt representations. Then, we investigate the model's ability to tackle another real-world problem: online CPD. We test it on two datasets: the USC human activities dataset (USC-HAD) \footnote{https://sipi.usc.edu/had/} \cite{mi2012uschad} and intra-operative electrocochleography (ECochG) data \cite{campbell2016}. Finally, using ECochG data, we demonstrate the model's potential in addressing the early detection problem.

   The main contributions of our work are as follows.
    \begin{itemize}
      \item We investigate various data augmentation methods for different types of time series data, and propose an efficient augmentation strategy that works well with proposed model for online CPD problems.
      \item We develop a simple combination of Transformer and TCN to efficiently capture both global and local properties of time series data. We call this novel structure Temporal-Transformer.
      \item We adapt supervised contrastive learning for time series data and simplify the implementation of SupCon Loss function to improve the model's efficiency.
      \item By combining all three components, we introduce SCOTT to efficiently and effectively learn representations from labelled time series data.
      \item We experimentally show that our proposed SCOTT performs well on a large range of real-world datasets, and is also reliable and efficient when applied to online CPD tasks (around 98\% area under precision-recall curve). These indicate its potential usage for real-world problems.
    \end{itemize}

  \section{Related Works} \label{lit}

   In this section, we first discuss contrastive learning and its usage on time series data. Next, we briefly discuss SOTA models for TSC and online CPD.
   \subsection{Contrastive Learning}
    Recently, self-supervised contrastive learning has gained increasing attention thanks to the success of several works. SimCLR \cite{chen2020simple} introduces contrastive loss to learn representations of visual inputs by maximizing the agreement between two augmentations of the same instance. BYOL \cite{grill2020byol} builds two networks to interact and learn from each other, and achieves SOTA results without any negative samples. MoCo \cite{he2020momentum} introduces an unsupervised learning framework to learn visual representations. On the other hand, supervised contrastive learning has been proposed to utilize label information in image representation learning \cite{khosla2020supervised}.

    The successes of self-supervised contrastive learning on CV and NLP tasks have encouraged researchers to apply it to time series tasks in recent years. TS-TCC \cite{eldele2021tsrl} introduces a temporal contrasting module together with a contextual contrasting module, to learn robust and useful representations of time series. TS-CP$^2$ \cite{deldari2021tscpd} utilize InfoNCE loss and a negative sampling method \cite{sohn2016negsamp} for time series CPD tasks. TimeCLR \cite{yang2022timeclr} combines the strengths of SimCLR and Inception Time \cite{ismail2020itime} to address the TSC problem. TS2Vec \cite{yue2022ts2vec} proposes a hierarchical contrasting to learn universal representations for time series data and achieves SOTA results. Apart from temporal domain, information residing in frequency domain  also facilitates to improve time series representation learning \cite{yang2022unsupervised,liu2023temporal}. Note that such existing methods focus on self-supervised or unsupervised learning, as opposed to our supervised learning basis.

   \subsection{Time Series Classification and Change Point Detection}
    There exist numerous algorithms to address TSC problems. Traditional methods measure the similarity between time series using distance measures \cite{goreck2013dervat}, or extract useful patterns from subseries, to determine the class of the time series \cite{ye2011shapelet,hills2014st}. Dictionary methods such as Bag-of-SFA-Symbols (BOSS) \cite{schafer2015boss} focus on the frequency of subseries. Ensemble methods like COTE, HIVE-COTE, TS-CHIEF, and Rocket combine several techniques and achieve great success on TSC tasks \cite{lines2018hivecote,shifaz2020chief,dempster2020rocket}. Apart from these, Deep Learning (DL) methods like Recurrent Neural Networks, Convolutional Neural Networks, and Transformer have also be applied for TSC \cite{pham2021lstmphy,ismail2020itime,chen2022tcntrans}.

    As another important time series problem, CPD's solutions use both supervised and unsupervised learning frameworks. The supervised methods recognise it as a classification problem (binary or multi-class) \cite{feuz2015activity,zheng2008,reddy2010}. Unsupervised algorithms include probability density, probabilistic, clustering, and deep learning methods. Among these, Relative Unconstrained Least-Squares Importance Fitting (RuLSIF) \cite{yamada2011ruslif}, Bayesian Online CPD (BOCPD) \cite{adams2007bayesian}, and Gaussian Process \cite{gpts2013} are three SOTA methods for the online CPD task.

  \section{Proposed Methodology} \label{mod}
   In this section, we first discuss data augmentation methods for varying types of time series data and introduce our augmentation strategy for online CPD problem. Secondly, we outline components of the proposed learning framework. Finally, we present the objective function and explore its simplification. An overview of the SCOTT system is provided in Figure \ref{scott}.
   \begin{figure}[tbp]
       \centering
       \includegraphics[scale=0.4]{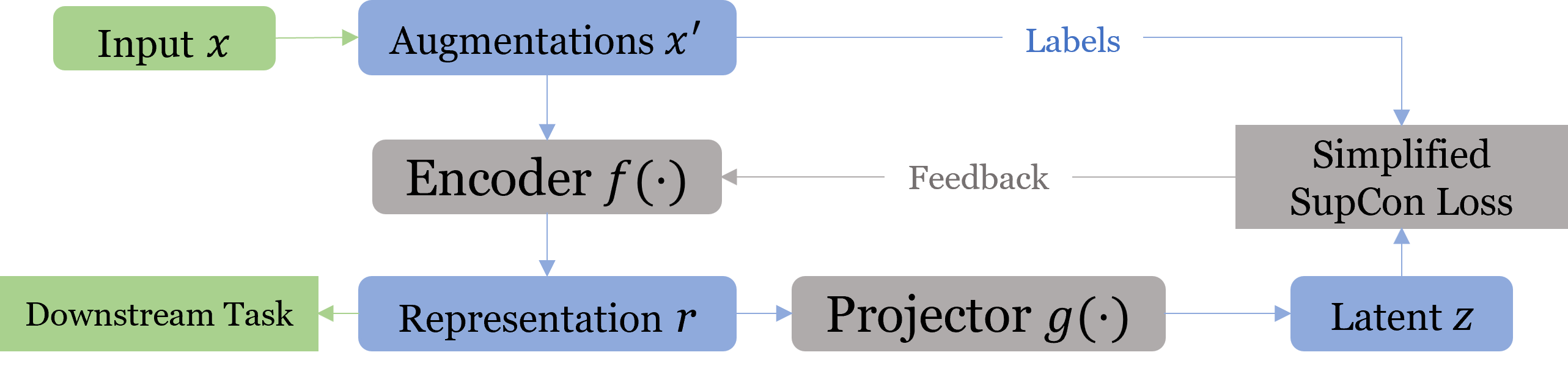}
       \caption{SCOTT Overview}
       \label{scott}
   \end{figure}
   \subsection{Data Augmentation}
    Contrastive learning typically requires augmented data to learn change-invariant representations. Augmentations such as jittering, permutation, scaling, and warping have been used in the literature for time series data \cite{iwana2021tsaugsuv,rashid2019tsaugar,eldele2021tsrl}, based on the nature of the dataset. For example, jittering, or adding noise to the original data, assumes that noise is common in the target data. Although this may be true for signal datasets like ECG, it may not be the case for value sensitive datasets like stock prices. Choosing an appropriate augmentation method is crucial as the wrong augmentations could alter the inherent properties of a dataset and cause errors in the developed models.

    Assuming we have a univariate time series $\vx = \{ \evx_1, \evx_2, \dots, \hdots, \evx_t \}, \evx_t\in \R, t \in \sN $, its $h$ segments $seg_{i}, i=\{1,2, \hdots h \}$, the noise $\rve = \{\erve_1, \erve_2, \hdots, \erve_t \}, \erve_i \sim \mathcal{N}(0, \sigma^2)$, and the scale factor $ \vs = \{ \evs_1, \evs_2, \hdots, \evs_t \}. \evs_i \in (0,1) $. A distorted segment $seg_{i_{dis}}$ is a segment expanded or shrunk via interpolation. $\{m_1, m_2, \hdots, m_h \} \in\{1,2,\hdots, h\}$ are random indices. We define several possible augmentation formulae:
    $$\begin{array}{lll}
      &\textit{\textbf{jittering}}: &\{\evx_1+\erve_1, \evx_2+\erve_2, \hdots, \evx_t+\erve_t \} \\
      &\textit{\textbf{scaling}}: & \{\evs_1*\evx_1, \evs_2*\evx_2, \hdots, \evs_t*\evx_t \} \\
      &\textit{\textbf{warping}}: & \{seg_{1}, \hdots,  seg_{j_{dis}}, \hdots, seg_{h} \} \\
      &\textit{\textbf{permutation}}: & \{seg_{m_1}, seg_{m_2} \hdots, seg_{m_h}\}
    \end{array}$$
    Note that warping will generate new time series with different length. We deal with this issue by interpolating the resulting series into the original length.

    We first select appropriate augmentation for datasets from the UCR archive. The selection is made based on the recommendations from several previous works \cite{iwana2021tsaugsuv,wen2021tsaug,rashid2019tsaugar} and a study we conduct on nine development datasets (details can be found in Section \ref{select}).

    Furthermore, based on these existing methods, we propose a novel augmentation strategy for the online CPD task. First, to operate in online scenario, we sample the data using a sliding window of length $\lambda$, moving the window forward by one step each time a new data point becomes available. To ensure that every point in the time series can be sampled, we add $\lambda-1$ points to the start of the time series.
    These extracted slices are labelled based on whether the last point is in the change state or not. Thus, we can assume that only the tail of a slice contains important information for detecting change points, which permits us to alter the initial part of a slice without changing any key features. As such, we use jittering and permutation on the first $\lambda - \beta$ points of a slice while keeping the last $\beta$ time points unchanged. Jittering is selected to enhance the model's robustness to noise, and permutation is used to reduce the dependence of the initial part of the time series and force the model to learn from the tail (They are also considered the most effective for `sensor' data according to the aforementioned augmentation study). Figure \ref{aug} shows an example of this sample-augmentations procedure on ECochG data.
    \begin{figure}[tbp]
      \centering
      \includegraphics[scale=0.45]{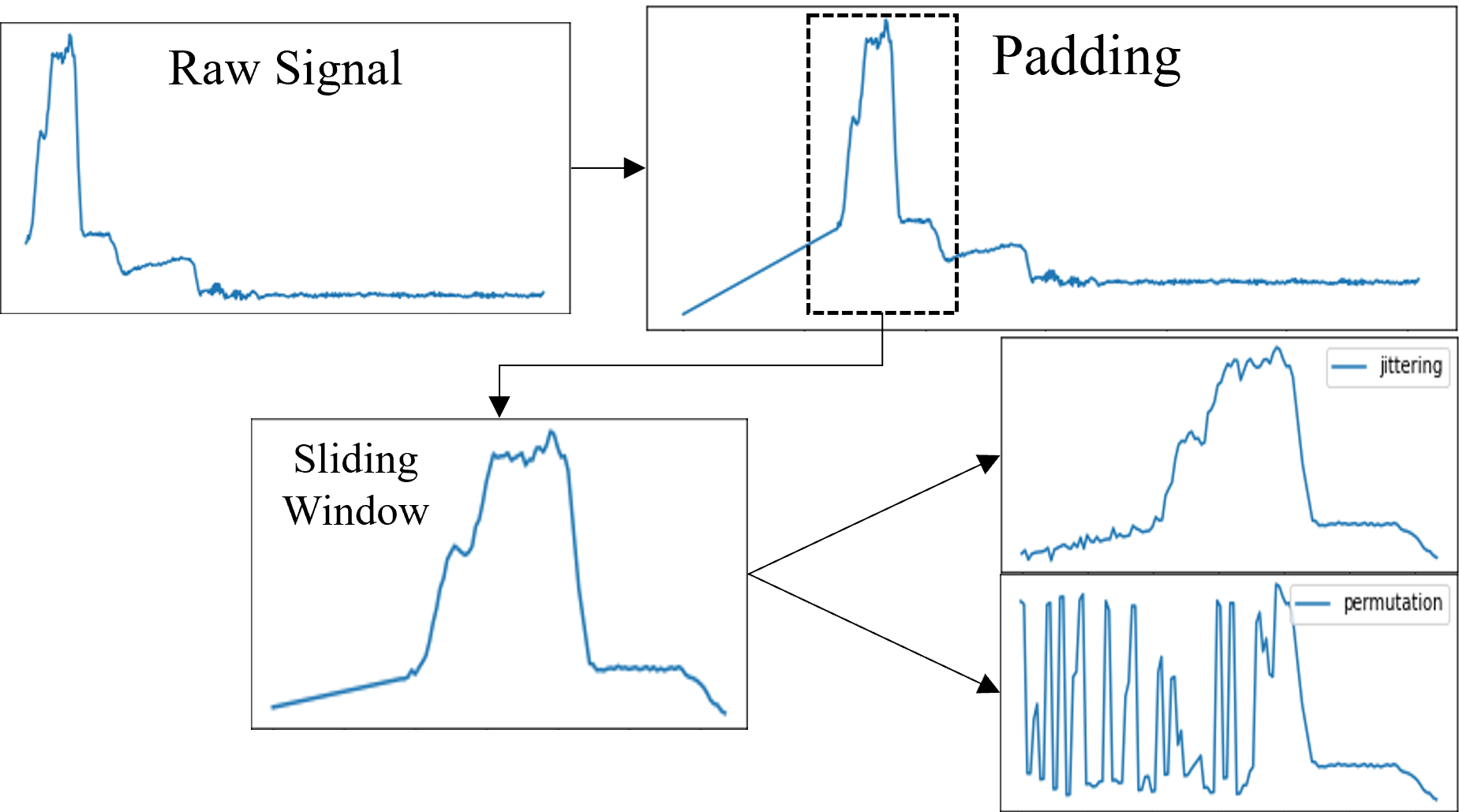}
      \caption{Augmentation for Online CPD}
      \label{aug}
    \end{figure}
    Additionally, in the case of unbalanced datasets, we oversample the under-represented class in the training data. \cite{yap2014over} argue that oversampling via duplicating the existing data could cause overfitting. To overcome this, we oversample the data via the aforementioned two types of augmentations.

   \subsection{Temporal-Transformer}
    After data has been augmented, we develop an encoder $f(\cdot)$ to extract the representations $\vr$ from them. The Transformer has shown its ability to capture global properties and been widely used to extract long term dependencies from different types of data \cite{vaswani2017attention,shen2019tensorized}. Convolution neural Networks (CNNs), on the other hand, have been proven to perform well on local processing via decomposition of depth-wise and point-wise convolution \cite{howard2017mobile,xie2017lgcnn}. Temporal Convolutional Networks (TCN), consisting of dilated causal convolutions, enhance a CNN's capability to learn temporal dependencies in sequential data like time series \cite{lea2016tcn}. To learn both global dependencies and local interactions simultaneously, one straightforward way is either stacking a Transformer upon TCN or vice versa \cite{chen2022tcntrans}. However, such direct combinations introduce unnecessary computational costs.

   Transformer's ability of extracting global properties is mostly contributed to by its self-attention mechanism \cite{zhang19dsagan,shen2019tensorized}. Thus, to capture the global dependencies of time series, we only need the attention matrix from the original Transformer.  Afterwards, we substitute the feed forward part of the original Transformer with the dilated casual convolutions of TCN. With this simple fusion of Transformer and TCN, we propose the encoder: Temporal-Transformer that efficiently learns both global and local attributes simultaneously. The structure of the encoder is shown in Figure \ref{temtrans}, the self-attention layer learns global dependencies of the input series and the following dilated causal convolutions will improve the learning outcome by capturing local interactions. The residual connections, which bypass these two blocks, help with preserving the important information from original data \cite{he2015deep}.
    \begin{figure}[tbp]
        \centering
        \includegraphics[scale=0.4]{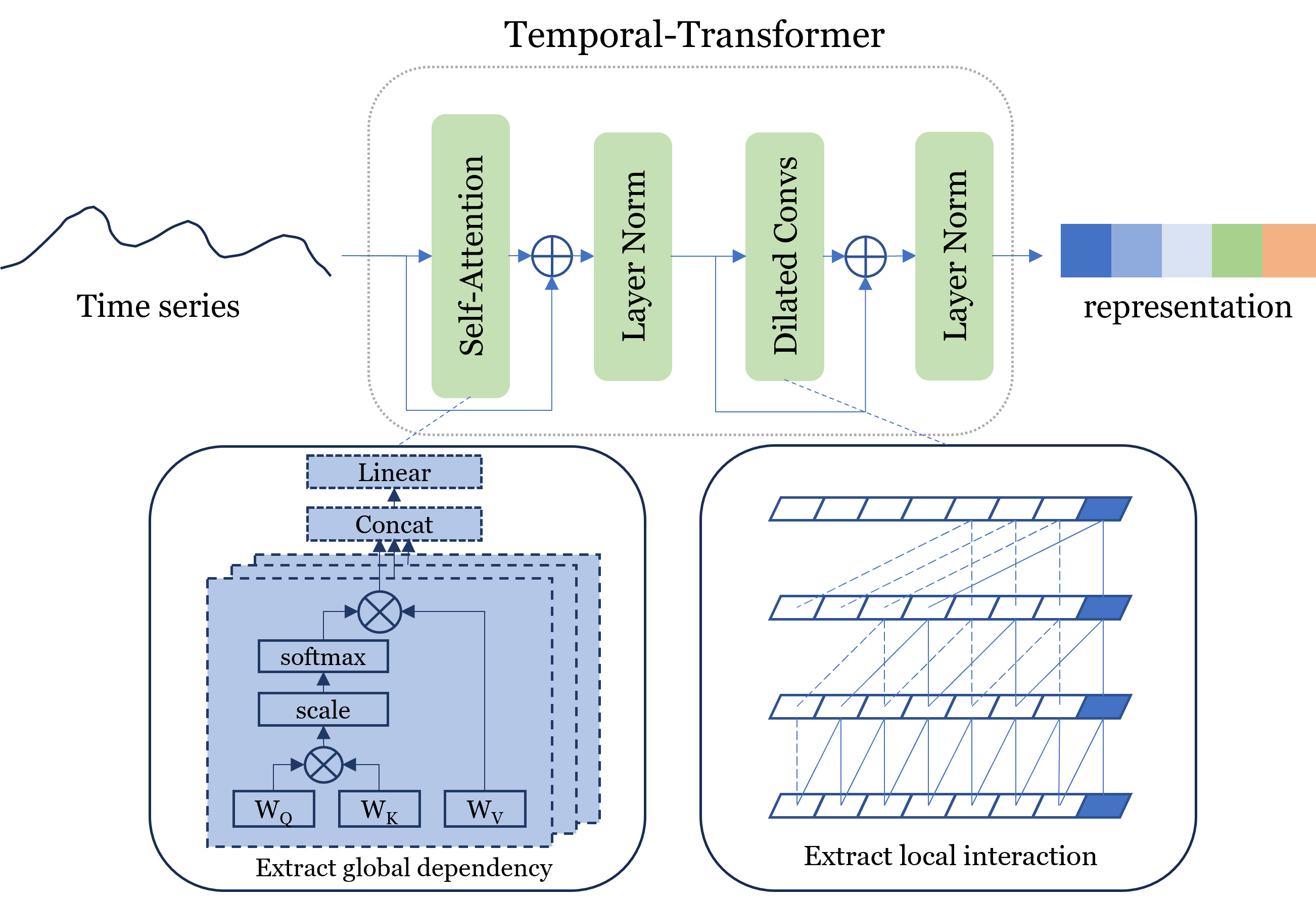}
        \caption{Temporal-Transformer}
        \label{temtrans}
    \end{figure}

   The significance of a projector network $g(\cdot)$ in SimCLR model is discussed in \cite{chen2020simple}. It maps the representation $\vr$ to a lower-dimensional vector $\vz$, making it applicable for the objective functions. We follow their design, add a Fully Connected Network as our projector.

   \subsection{Objective Function}
    In SimCLR \cite{chen2020simple}, a batch of $bn$ samples is augmented into pairs resulting in a dataset of $2bn$. Each pair itself is considered to be positive and the remaining $2bn-2$ points are treated as negative samples. Assuming $i, j \in \sA = \{ 1,2,3, \hdots , 2bn \}$, is the index of an augmented pair, let $sim(z_i, z_j)$ denote the similarity between the vectors $z_i$ and $z_j$. This can be measured as cosine similarity, dot product etc. Then, the objective function: `NT-Xent' loss $\mathcal{L}$ has the form:
    \begin{equation} \label{eq:1}
      \mathcal{L} = -\sum_{i,j \in \sA} \log \frac{\exp(sim(z_i, z_j)/\tau)}{\sum_{k\in \sX} \exp(sim(z_i, z_k)/\tau)}
    \end{equation}
    where $\sX = \sA \backslash \{ i \}$ and $\tau$ is the temperature parameter that determines the probability distribution.

    With equation (\ref{eq:1}), the model can maximize the agreement between positive pairs and learn the representations of the original data. One problem in this function is that different samples in the same class could be regarded as negative pairs. To handle such issues, the label information $y_{\omega}(\omega \in \sA)$ is introduced into the supervised contrastive loss $\mathcal{L}^{sup}$ \cite{khosla2020supervised}:
    \vspace{-3mm}
    \begin{multline}\label{eq:2}
        \mathcal{L}^{sup} = \\ -\sum_{i\in \sA} \frac{1}{|\sP(i)|} \sum_{p\in \sP(i)} \log \frac{\exp(sim(z_i, z_p)/\tau)}{\sum_{k \in \sX} \exp(sim(z_i, z_k)/\tau)}
    \end{multline}
    where $\sP(i) = \{ p\in \sX | y_i = y_p \}$ is the set of indices of the positive samples in the batch and $|\sP|$ is the cardinality.

    We simplify the implementation of (\ref{eq:2}) by calculating all instances together. The original loss function calculates the cost separately on different views (augmentations) of one instance, and in the final step averages these costs. In supervised contrastive loss, since the class label is used for the calculation, the instances from the same class can be calculated together according to the formula. Thus, there is no need to separate the calculation of different views. In $Proposition$ \ref{prop1}, we prove this simplification does not change the effect of the original function. Additionally, this simplification can slightly reduce computation time. We provide a detailed study of this in Section \ref{time}.

    \begin{proposition} \label{prop1}
      For the calculation of $\mathcal{L}^{sup}$, taking each augmentation as one instance of its own class has the same effect as calculating them separately and averaging at the end.
    \end{proposition}

    \begin{proof}
      Suppose we have $n_v$ views for each instance $x_i$ from a batch of size $bn$, where $i \in \sI=\{1, 2, \hdots, bn\}$, $\sY = \sI \backslash \{ i \}, \sY_q = \{z_{i_l} | l\neq q \}$. Then, the $\mathcal{L}^{sup}$ with separate views can be written as:
      \begin{eqnarray}\label{eq:3}
        \mathcal{L}^{sup} &=& -\sum_{i\in \sI} \frac{1}{|\sP(i)|} \sum_{p\in \sP(i)} \nonumber \\
        && \hspace*{2cm}\log \frac{\exp(sim(z_i, z_p)/\tau)}{\sum_{k \in \sY} \exp(sim(z_i, z_k)/\tau)}\nonumber \\
         &=& -\sum_{i\in \sI} \frac{1}{n_v} \sum_{j=1}^{n_v} \frac{1}{|\sP(i_q)|} \sum_{p\in \sP(i_q)} \nonumber \\
         && \hspace*{2cm} \log \frac{\exp(sim(z_{i_q}, z_p)/\tau)}{\sum_{k \in \sY_q} \exp(sim(z_{i_q}, z_k)/\tau)}\nonumber \\
         &=& -\sum_{i\in \sI} \sum_{j=1}^{n_v} \frac{1}{n_v|\sP(i_q)|} \sum_{p\in \sP(i_q)}  \\
         && \hspace*{2cm} \log \frac{\exp(sim(z_{i_q}, z_p)/\tau)}{\sum_{k \in \sY_q} \exp(sim(z_{i_q}, z_k)/\tau)}\nonumber
      \end{eqnarray}

      Assuming $i^* \in \sI^* = \{ 1_1, 1_2, \hdots, 1_{n_v}, \hdots, bn_{n_v} \}$, $\sY^* = \sI^*\setminus\{i^*\} $. We have:
      \begin{eqnarray} \label{eq:4}
        (\ref{eq:3}) &=& -\sum_{i^*\in \sI^*} \frac{1}{n_v|\sP(i^*)|} \sum_{p\in \sP(i^*)} \nonumber\\
        && \hspace*{2cm}\log \frac{\exp(sim(z_{i^*}, z_p)/\tau)}{\sum_{k \in \sY^*} \exp(sim(z_{i^*}, z_k)/\tau)}\nonumber \\
         &\sim& -\sum_{i^* \in \sI^*} \frac{1}{|\sP(i^*)|} \sum_{p\in \sP(i^*)} \\
         && \hspace*{2cm}\log \frac{\exp(sim(z_{i^*}, z_p)/\tau)}{\sum_{k \in \sY^*} \exp(sim(z_{i^*}, z_k)/\tau)}\nonumber
      \end{eqnarray}

      Note that (\ref{eq:4}) is the same as the formula for taking each augmentations as one instance of its own class, thus the proposition stands.
    \end{proof}

  \section{Experiments} \label{exp}
   \subsection{Datasets}
    \textbf{UCR Time Series Classification Archive} We evaluate SCOTT's performance on the TSC problem using 45 UCR datasets that cover all types in the archive: image, sensor, motion, spectro, traffic, device, simulate, audio and ECG data. Nine of these datasets, namely: BeetleFly, ChinaTown, ChlorCon, ECG200, GunPoint, InsecWing, Meat, PowerCon, and Wafer (one from each type), are used to develop the model, specifically for hyperparameter tuning. Details of these 45 datasets can be found in \cite{scottsup}.\\
    \textbf{USC-HAD}: This human activity dataset contains 12 activities recorded in five trials with $6-$axes. We follow similar pre-processing steps from the previous work \cite{deldari2021tscpd}: We use one accelerometer axis of two activities (one of which is the change state), from the first five human subjects. Randomly mixing them up and implementing the aforementioned sampling method, we get over $200,000$ time series.\\
    \textbf{ECochG}: ECochG data consists of the response of the inner ear to sound played during Cochlear Implant (CI) surgery. This surgery implants an electrode in the inner ear to simulate the function of the cochlea and is used to supplement the residual natural hearing of the hearing imapaired. However, there is a high probability that patients would lose their natural hearing due to trauma during the procedure. A previous study \cite{campbell2016} found that a $30\% +$ drop in `Cochlear Microphonic' (CM - one component of ECochG) amplitude may reflect damage to the natural hearing of the patient. This motivated researchers to detect these `traumatic drops' in CM amplitude to prevent trauma during CI surgery \cite{sudanthi2022ecochg}. In this dataset, we have records from 78 patients (78 time series) with varying length, ranging from 129 to 991 temporal points. By implementing the aforementioned sampling method, we get 30491 time series in total. We consider traumatic drops (as labelled by an expert) to be the change state and all others to be the non-change state.


   \subsection{Baseline Models \& Experimental Setup}
    For TSC baselines, we select traditional SOTA models with the leading results (from the website \footnote{https://www.timeseriesclassification.com/} and \cite{dempster2020rocket}): Shapelet (ST), BOSS, HIVE-COTE (HCOTE), TS-CHIEF, Rocket and Inception Time (ITime). We also add two SOTA Representation Learning (RL) models: SimCLR and TS2Vec. For fair comparison, an off-the-shelf Multi-layer Perceptron (MLP) classifier is used to classify the learnt representations from RL models including ours. We also list results from MLP without any RL models. For the online CPD problem, we select BOCPD, RuLSIF, and TS2Vec. We also list the results of \cite{sudanthi2022ecochg}, where a tree ensemble was found to perform well for the ECochG task. We follow the published code to implement the RL models (with necessary modifications).

    Nine selected datasets are utilized for hyperparameter tuning of our proposed model. To evaluate the learned representations, we incorporate an MLP block to enable classification applicability. Hyperparameters are randomly searched, selecting the optimal combination. The encoder comprises a temporal transformer with three heads of size 256, followed by three dilated convolution layers (dilation rates: 1, 4, and 16, respectively) each with 256 filters and a kernel size of 4. The projector is an MLP with hidden layers of 128 and 64 units, employing a 0.3 dropout rate and producing an 8-dimensional vector. Activation functions are 'relu', with a supervised contrastive loss temperature of 1.0. The initial learning rate is set at 0.001, halving when loss stagnates, over 300 training epochs with a batch size of 128. For classification, a three-layer MLP classifier is employed, consisting of hidden layers with 256 and 64 units, 'relu' activation, and 0.4 dropout, followed by an output layer with 'softmax'. These classifiers undergo training for 200 epochs with a batch size of 64, implementing an early-stopping mechanism based on validation loss. Identical configurations are maintained in the baseline MLP models for equitable comparisons. The code can be found in \cite{scottsup}.

    In the UCR archive, all the datasets have already been split into training and testing sets. We use these subsets for the training and evaluation of our models respectively.

    In ECochG data, change points are defined based on the ratio between peak and current value. When this ratio reaches a certain threshold (30\%), a change state occurs. Thus, we define the two parameters, window length ($\lambda$) and tail length ($\beta$), based on ECochG data. We observe that the furthest distance from a peak to its nearest trough is 50 points. Hence, we set $\beta=50$. The selection of $\lambda$ is empirical and addresses the need to have enough instances in the leading part of the time series to provide successful augmentations. As such, we set the leading part of the time series to be twice as long as the tail:  $\lambda-\beta=2\beta=100$, resulting in $\lambda=150$. As a rising trend is considered to be the non-change state, we pad the first 149 points of the ECochG data using a rising trend. To ensure the functionality of standard CPD algorithms (BOCPD and RuLSIF), in this scenario, we define a scoring system based on ``30\% drop" threshold to assist with final decision making. The Human activity data is normal (based on a hypothesis test \cite{ralph1973normtest}), and the difference between the two states lies in the variance. We sample and augment the human activity data using the method discussed in Section \ref{mod} and pad the start of the time series with 149 normally distributed points based on the parameters of the non-change state (we fit the data to a normal distribution and get the best fit with $\mathcal{N}(1, 0.5)$). We pad both the training and testing data to ensure each original point of the time series is considered in the calculations. We only augment the training dataset and leave the testing dataset unchanged in order to preserve its integrity.

    The experiments are undertaken using a High Performance Computer (HPC) with Intel Xeon CPU E5-2650 v4 (2.20GHz) and NVIDIA P100 GPU (12G).

   \subsection{Experimental Results}
    \subsubsection{Selection of Data Augmentation Method} \label{select}
     We evaluate different data augmentation methods on nine development datasets by comparing accuracy obtained with them. We use the aforementioned augmentation methods to create one view of the data and keep the original data as another view. The results are shown in Table \ref{ucraug}.
        \begin{table}[bp]
           \centering
           \caption{Accuracy of different augmentations on development datasets}
           \scalebox{0.8}{%
           \begin{tabular}{|c|ccccc|}
               \hline
               \textbf{Datasets} & \textbf{Type} & \textbf{Warping} & \textbf{Permutation} & \textbf{Scaling} & \textbf{Jittering} \\ \hline \hline \xrowht{5pt}
               BeetleFly & IMAGE & \textbf{1.0000} & 0.8500 & 0.9500 & 0.9000 \\
               ChinaTown & TRAFFIC & 0.9825 & \textbf{0.9854} & 0.9796 & 0.9738 \\
               ChlorCon & SIMULATE & 0.8128 & 0.8805 & 0.8576 & \textbf{0.8974} \\
               ECG200 & ECG & 0.9000 & 0.8900 & 0.9300 & \textbf{0.9400} \\
               GunPoint & MOTION & \textbf{0.9733} & 0.9533 & 0.9667 & 0.9667 \\
               InsecWing & AUDIO & 0.6288 & 0.6217 & 0.6455 & \textbf{0.6566} \\
               Meat & SPECTRO & 0.9167 & \textbf{0.9667} & \textbf{0.9667} & \textbf{0.9667} \\
               PowerCon & DEVICE & 0.9944 & \textbf{1.0000} & \textbf{1.0000} & \textbf{1.0000} \\
               Wafer & SENSOR & 0.9958 & \textbf{0.9968} & 0.9958 & 0.9966 \\ \hline
           \end{tabular}}
           \label{ucraug}
        \end{table}
     We can see that for the image type time series, warping and scaling get the top two results, which suggests they could be suitable augmentations for other datasets of this type. For traffic and sensor data, permutation has the best results, while jittering and warping show second best results. Thus, they can be good options for these two types of time series data. For ECG, simulate, and audio data where noise is common, jittering is the best augmentation method. Both spectro and device get three best candidates (permutation, scaling, and jittering) as appropriate augmentations.

    \subsubsection{Time Series Classification}
     Due to space limit, we provide detailed evaluation results for TSC in \cite{scottsup}. Here, we show statistical result using Critical Difference Diagram \cite{demsar2006stateva} of each model's average ranking in Figure \ref{ucr_cd}. A black line connects models whose pairwise accuracy was not significantly different (as determined using Wilcoxon signed-rank tests with Holm correction).

      \begin{figure}[bp]
          \centering
          \includegraphics[scale=0.5]{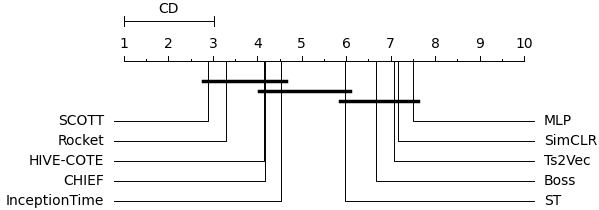}
          \caption{Critical Difference of model average ranking}
          \label{ucr_cd}
      \end{figure}

      We observe that our model's results are close to the SOTA models or better (for 23 datasets). Inception Time achieved the second best results with 10 wins. As to the average ranks of each model. SCOTT is slightly better than Rocket (2.90 compared to 3.29), and ranks first among all evaluated models. On the other hand, MLP shows the worst results. This suggests SCOTT is successful in extracting useful representations from time series data, that enables a relatively weak classifier to become competitive with the SOTA models. These observations highlight the potential of supervised contrastive learning for labelled time series representation learning and its downstream classification task.

      To better understand how SCOTT supports TSC, we visualize some datasets before (left) and after (right) representation learning via our model. We map each instance of a dataset onto a 2-dimensional vector using t-SNE \cite{vandermaaten08a} and visualize these instances using different colours according to their class. As shown in Figure \ref{tsne}, instances from different classes are more separable after representation learning. This simplifies the subsequent classification problem.
      \begin{figure}[tbp]
       \centering
       \includegraphics[scale=0.38]{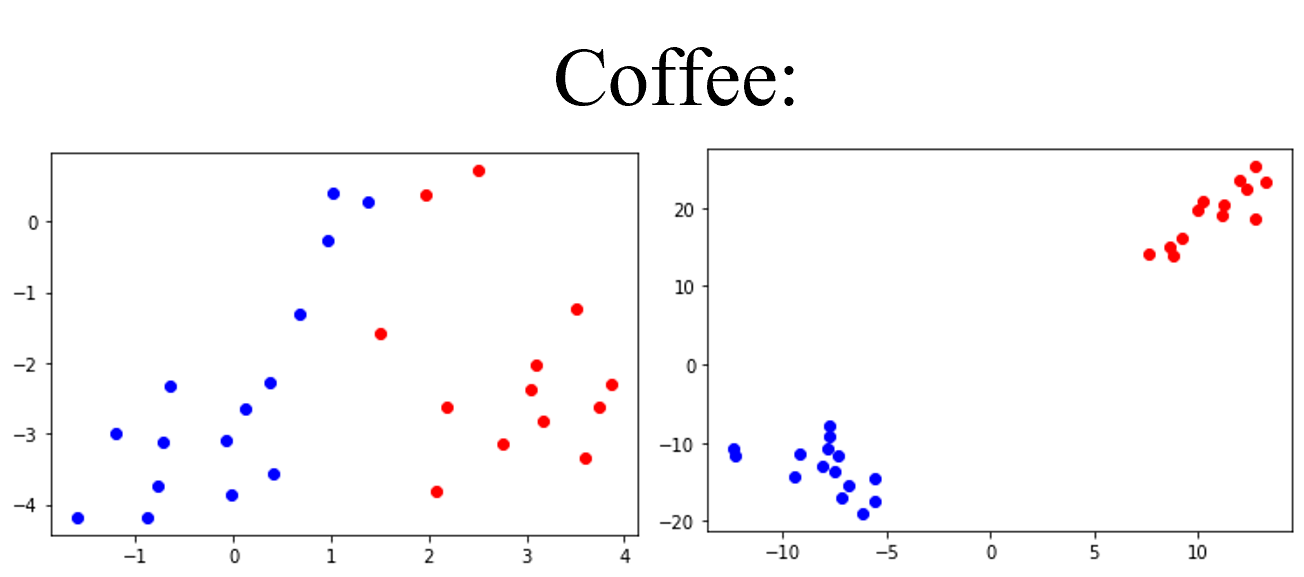}
       \includegraphics[scale=0.38]{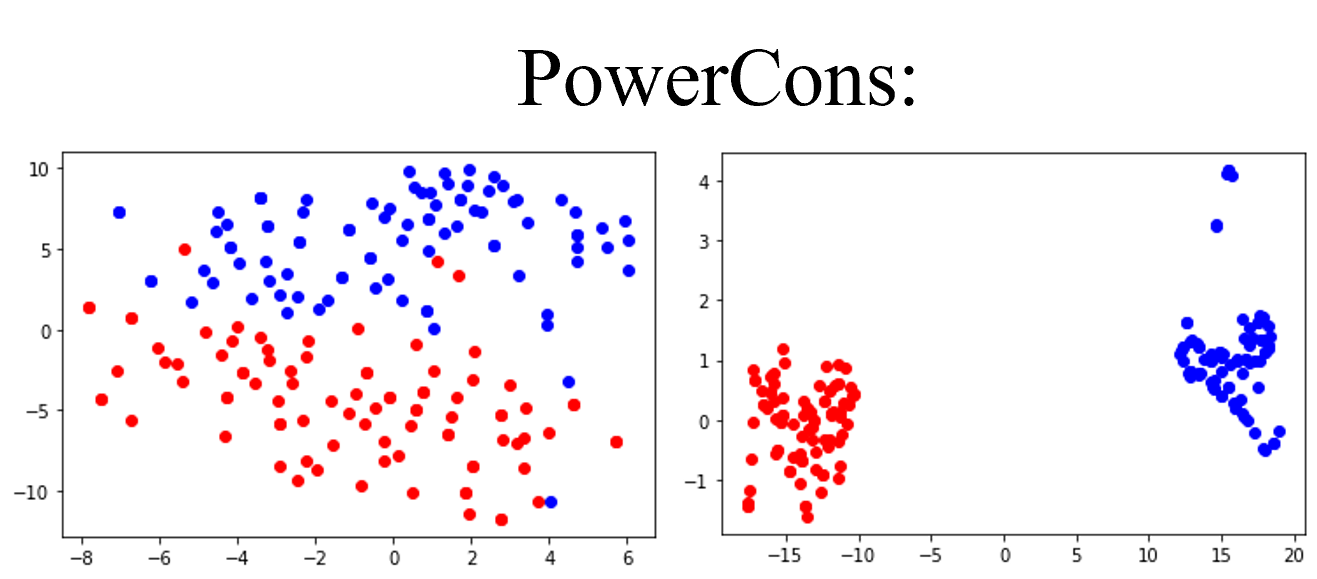}
       \includegraphics[scale=0.38]{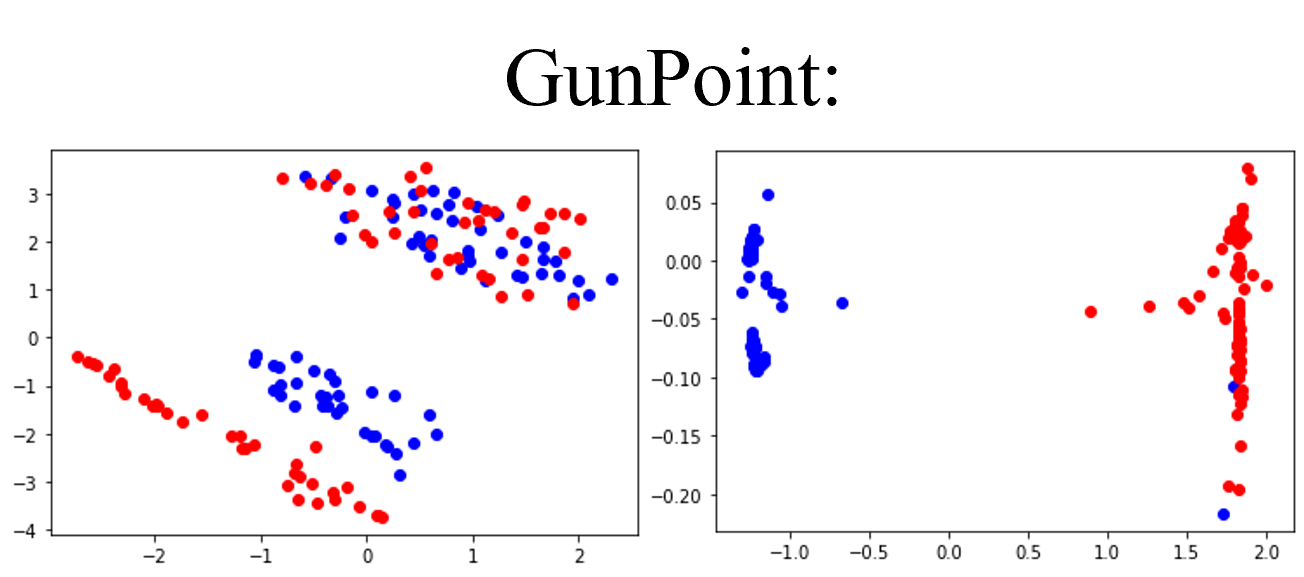}
       \includegraphics[scale=0.38]{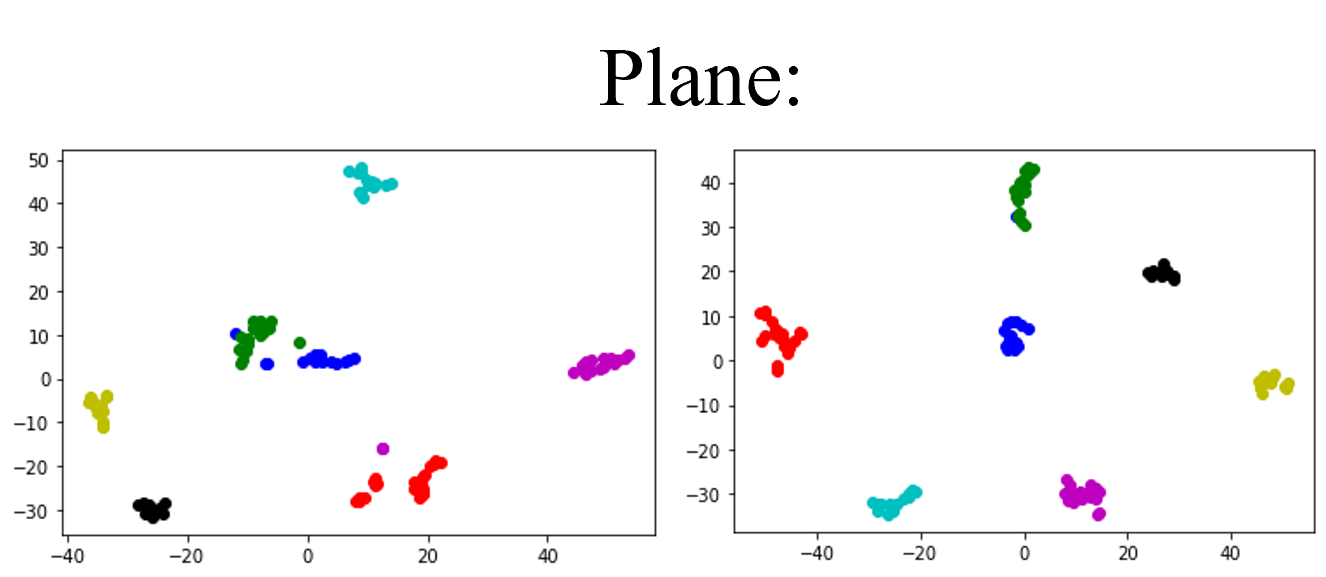}
       \caption{t-SNE Plots}
       \label{tsne}
      \end{figure}

    \subsubsection{Online Change Point Detection}
     We use streaming data to evaluate SCOTT for online CPD. In this type of problem, the model can only see data obtained on or before the current time. To address the CPD problem, we first sample the data using the aforementioned method and shuffle the training set. Thus, SCOTT should learn the representation based only on the current value (the last point in the current sliding window) and those of the previous points within the window. Once trained, SCOTT is expected to produce instructive representations that guide the classifier to make correct predictions. Then, we train the classifier using the learnt representations and their corresponding labels to predict whether the current point is in a change state. Here for convenience, we again use an MLP classifier to do the classification task.

     During inference, the system anticipates a continuous incoming signal. We simulate this process using the entire signal for each patient: We remove the first value in the current sliding window and appending the incoming data at the end. We annotate whenever the system returns a positive prediction. Figure \ref{cmexp} shows annotated ECochG signals of 4 new patients (unseen to the model) after this simulation, from which we can see the system also works well with these new data as most traumatic drops can be detected.
     \begin{figure}[tbp]
         \centering
         \includegraphics[scale=0.35]{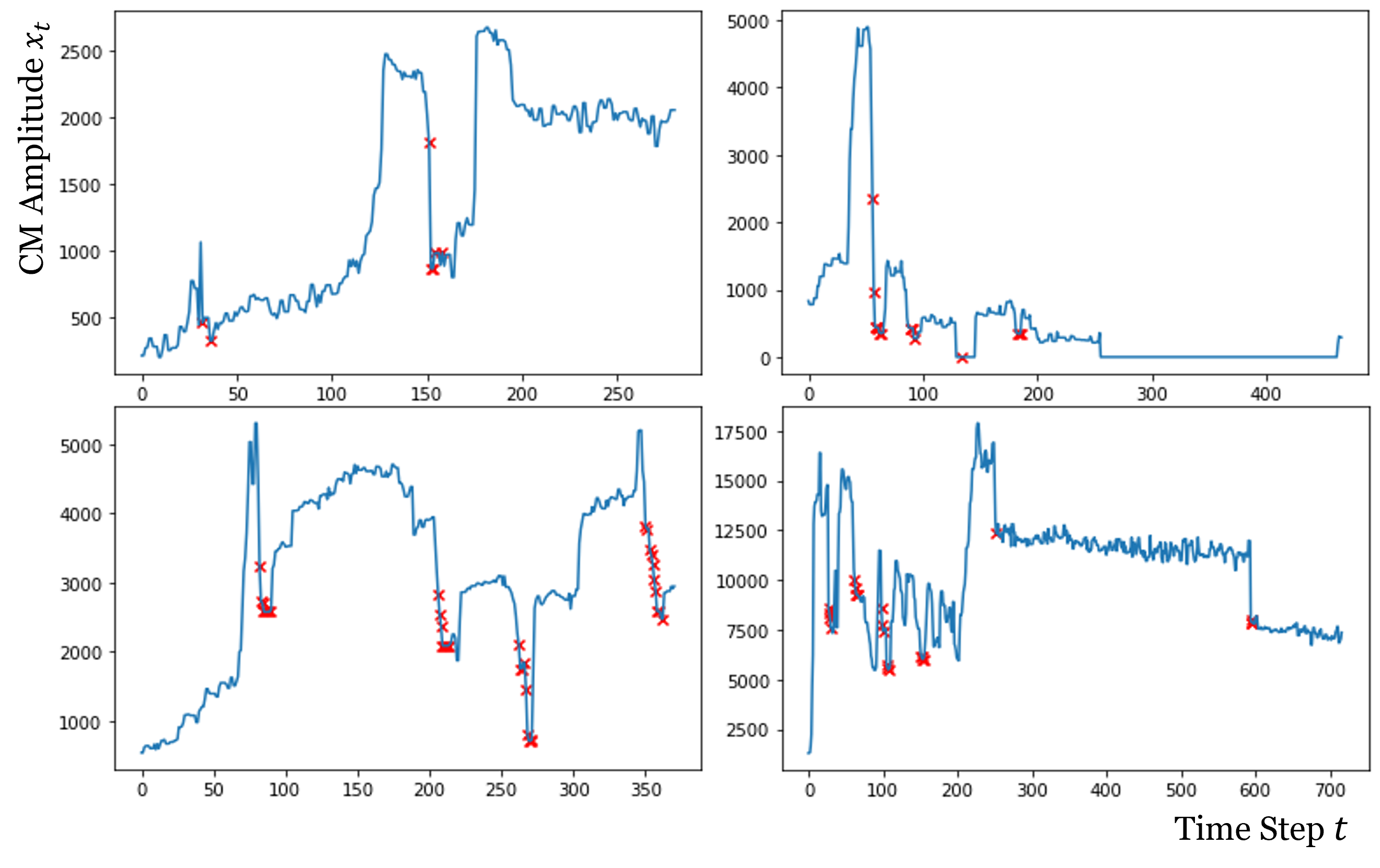}
         \caption{Annotated Patient Data: {\color{red}$\times$} stands for points that predicted to be inside traumatic drops (30\% drop).}
         \label{cmexp}
     \end{figure}

     The overall evaluation results are listed in Table \ref{ocpd_ecog}. According to the table, the proposed method shows the best performance for both datasets. The fact that our model performs well with respect to all metrics suggests that it can be used effectively in real-world applications.
     \begin{table}[tbp]
      \centering
      \caption{Online CPD Results}
      \scalebox{0.91}{%
        \begin{tabular}{|c|c|cc|}
            \hline \xrowht{5pt}
            \textbf{Models} & \textbf{Dataset} & \textbf{AUROC} & \textbf{AUPRC} \\ \hline\hline
            BOCPD & \multirow{5}{*}{ECochG} & 0.9369 & 0.7706 \\
            RuLSIF &  & 0.9283 & 0.6915 \\
            Tree Ensemble &  & 0.9423 & 0.6651 \\
            TS2Vec &  & 0.9813 & 0.7479 \\
            \textbf{SCOTT} & & \textbf{0.9989} & \textbf{0.9856} \\ \hline \hline \xrowht{5pt}
            BOCPD & \multirow{4}{*}{USC-HAD} & 0.9778 & 0.9594 \\
            RuLSIF &  & 0.9126 & 0.8913 \\
            TS2Vec & & 0.9802 & 0.9571 \\
            \textbf{SCOTT} &  & \textbf{0.9862} & \textbf{0.9678} \\ \hline
        \end{tabular}}
        \label{ocpd_ecog}
      \end{table}

    \subsubsection{Early Detection}
     One critical problem in real-world scenarios is detecting the change or abnormal states even before they occurred \cite{coughlin2021earlycovid,crosby2022early}. We briefly investigate our model's performance on the ECochG dataset under such requirements. We re-define the boundary of change states (traumatic drops), shifting it earlier by 1 to 30 steps. Then, we use the same pre-processing method to prepare the dataset. Note that under this new circumstance, the ``30\% drop" rule no longer stands. Hence, the standard CPD algorithms and the models in previous work \cite{sudanthi2022ecochg} are no longer applicable for the new problem (because part of these methods are derived from this 30\% rule). Therefore, we compare our model with three other classifiers, namely TS2Vec, Transformer (Trans), and MLP. We report Area under Precision-Recall curve and detailed Recall, Precision scores here as these metrics provide more representative assessment of model's effectiveness and overall performance. The results are shown in Figure \ref{early}.
     \begin{figure}[tbp]
         \centering
         \includegraphics[scale=0.55]{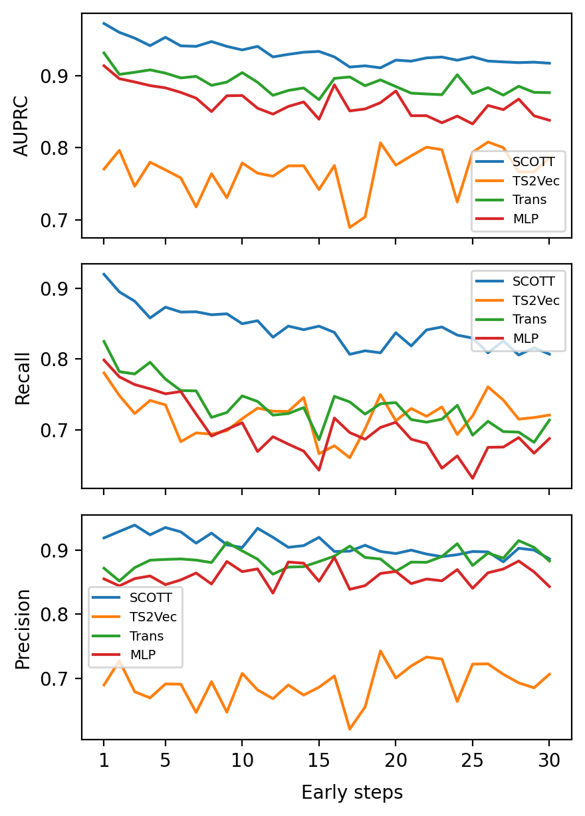}
         \caption{Early Detection Results}
         \label{early}
     \end{figure}
     All the models experienced performance decrease with respect to $Recall$, whereas $Precision$ is more stable comparatively. Note that the problem is to find out the ``traumatic drops" and those ``significant drops" are still inside positive class while we shifting the boundary. These ``significant drops" are, therefore, still distinct and easy to identify by the model, which could contribute to the stable $Precision$. However, ambiguities are introduced to the margin points. This makes the decision boundary become vague and lead to $Recall$'s decrease. Despite this model decay, we can observe our model still outperforms others and stays at a relative high level. This indicates our model's potential on early detection problems and highlights its future usage in CI surgery.

    \subsubsection{Execution Time} \label{time}
     We carry out separate experiments to investigate the comparative time cost when using original and simplified loss functions. We run the models on all the previous datasets. For each dataset we augment the instances into two views. Keeping the network configurations unchanged, we run 300 epochs with each loss functions. Here, the batch size of the simplified function $bn_{simp}$ is the original batch size $bn_{ori}$ multiplied by the number of views $n_v$, $bn_{simp} = bn_{ori} * n_v$. This is to ensure the number of processed instances are the same at each step for both functions. The total execution time of the two loss functions on all datasets are shown in Table \ref{comptime}. We can see our simplified loss function reduces the computation time.
     \begin{table}[tbp]
         \centering
         \caption{Total Execution Time (seconds)}
         \begin{tabular}{|c|cc|}
          \hline\xrowht{5pt}
             \textbf{Datasets} & \textbf{Original Loss} & \textbf{Simplified Loss} \\ \hline\hline \xrowht{5pt}
             UCR & 7306.99 & 7281.15 \\
             ECochG & 29837.28 & 29456.83 \\
             USC & 21457.47 & 21305.31 \\
             \hline
         \end{tabular}
         \label{comptime}
     \end{table}

     Furthermore, we investigate how variables such as the training set size and batch size affect execution time. The training set consists of randomly simulated data with length 256, augmented into two different views. We select training sizes of 512, 1024, 2048, 4096 and 8192 to train these datasets with five different batch sizes (here we refer to the original batch size $bn_{ori}$): 16, 32, 64, 128 and 256. We show the time difference between two loss functions ($original-simplified$) in the right-hand side of Figure \ref{td_var}.
     \begin{figure}[tbp]
         \centering
         \includegraphics[scale=0.35]{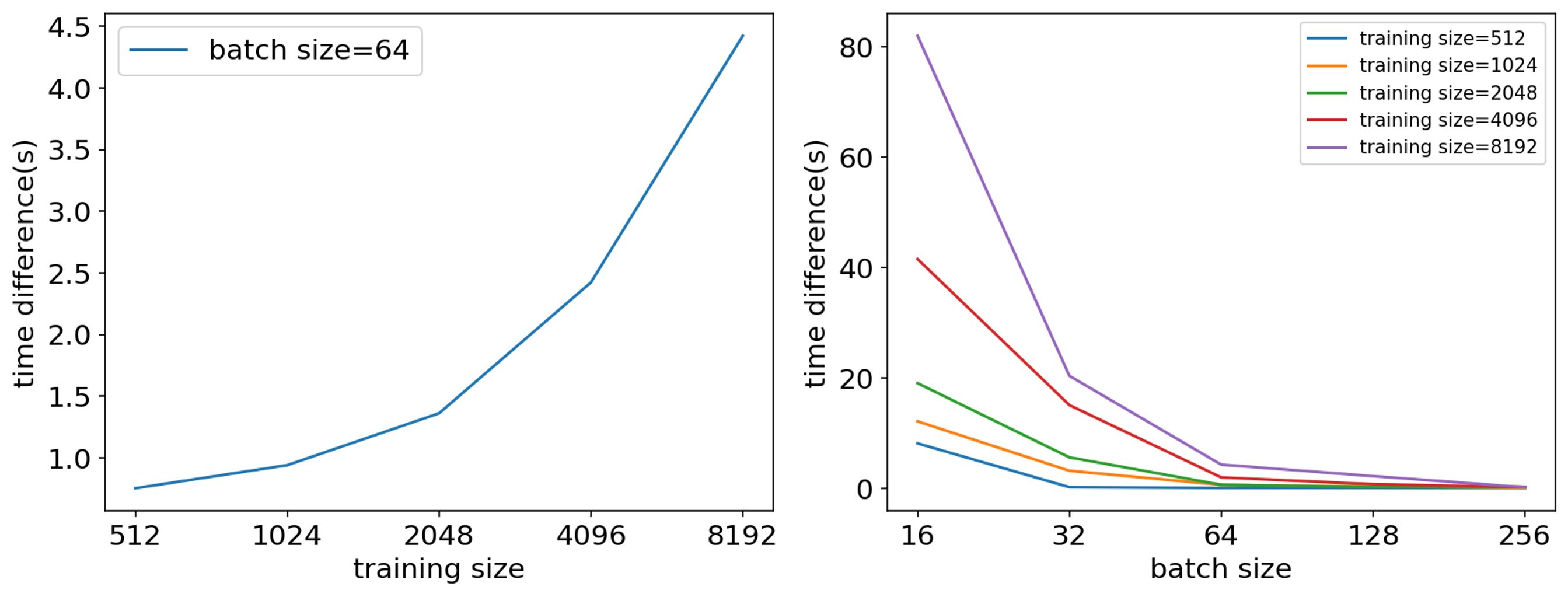}
         \caption{Time differences VS Training size and Time difference VS Batch size}
         \label{td_var}
     \end{figure}
     The left-hand side shows the relationship between time difference and training size, with the batch size fixed at 64. From these two figures, we can conclude that, as the training size grows, the computation time saved by using the simplified function increases. When the batch size is increased, the difference in computation time decreases. However, as the batch size increases, the memory requirement also increases. Therefore, when selecting a batch size, a balance between computation time and memory requirements has to be reached depending on the hardware availability.

    \subsubsection{Ablation Study}
     Using ECochG data, we investigate the contribution from each component of SCOTT. To this end, we train an MLP without encoding, a Temporal Transformer (TT) with the traditional Crossentropy Loss (TT+CE), and TT with SupCon Loss (SCOTT). Comparison of each two consecutive models show the contributions of the encoder and the loss function respectively. We train these models on the original ECochG data, augmented data (denoted as `Aug'), and over-sampled data (via augmentation, denoted as `Ov+Aug').
     \begin{table}[tbp]
      \centering
      \caption{Ablation Study}
      \scalebox{0.91}{%
      \begin{tabular}{|c|ccc|}
        \hline \xrowht{5pt}
        \textbf{Components} & \textbf{Accuracy} & \textbf{AUCROC} & \textbf{AUCPRC} \\ \hline\hline \xrowht{5pt}
        MLP & 0.9706 & 0.9824 & 0.8521 \\
        TT+CE & 0.9851 & 0.9950 & 0.9407 \\
        SCOTT & 0.9912 & 0.9923 & 0.9783 \\ \hline\hline \xrowht{5pt}
        Aug+MLP & 0.9757 & 0.9866 & 0.8870 \\
        Aug+TT+CE & 0.9853 & 0.9929 & 0.9394 \\
        Aug+SCOTT & 0.9902 & 0.9979 & 0.9762 \\ \hline\hline \xrowht{5pt}
        Ov+Aug+MLP & 0.9796 & 0.9903 & 0.9083 \\
        Ov+Aug+TT+CE & 0.9858 & 0.9943 & 0.9418 \\
        Ov+Aug+SCOTT & \textbf{0.9926} & \textbf{0.9989} & \textbf{0.9856} \\ \hline
      \end{tabular}}
      \label{abstd}
    \end{table}

     The results are shown in Table \ref{abstd}. From the first section of the table (rows 1-3), we see that most metrics are better for TT and SCOTT. The most important metric - AUPRC, which gives a more complete view of the performance - is better for TT based models indicating that they have better feature extraction and representation learning capabilities. It also shows that supervised contrastive learning is able to improve these capabilities. By comparing the three sections, we can also draw the conclusion that our proposed augmentation and oversampling methods both improve the model performance.

  \section{Conclusion} \label{con}
   In this paper, we proposed a novel model to learn time series representations with accessible label information. First, we used flexible augmentation methods for different types of data. Especially, for the online CPD application, we designed a data sample-augmentation method. Second, we adapted SupCon Loss together with Transformer and TCN to fit time series data and learn better representations. Finally, using the learnt representations, we used a basic classifier to effectively address TSC tasks. We showed experimentally that the proposed model performed well when compared to SOTA models. We also established its effectiveness on online CPD tasks. This opens the door for research on TSC and online CPD problems using supervised contrastive learning.

  \section*{Acknowledgment}
   This work is partially supported by NHMRC Development Grant Project 2000173 and The University of Melbourne’s Research Computing Services and the Petascale Campus Initiative.

\bibliographystyle{IEEEtran}
\bibliography{supconref}

\end{document}

%% file: math_commands.tex
\usepackage{amsmath,amsfonts,bm}

\def\eqref#1{equation~\ref{#1}}

\def\1{\bm{1}}

\def\rve{{\mathbf{e}}}

\def\erve{{\textnormal{e}}}

\def\vr{{\bm{r}}}
\def\vs{{\bm{s}}}

\def\vx{{\bm{x}}}

\def\vz{{\bm{z}}}

\def\evs{{s}}

\def\evx{{x}}

\DeclareMathAlphabet{\mathsfit}{\encodingdefault}{\sfdefault}{m}{sl}
\SetMathAlphabet{\mathsfit}{bold}{\encodingdefault}{\sfdefault}{bx}{n}

\def\sA{{\mathbb{A}}}

\def\sI{{\mathbb{I}}}

\def\sN{{\mathbb{N}}}

\def\sP{{\mathbb{P}}}

\def\sX{{\mathbb{X}}}
\def\sY{{\mathbb{Y}}}

\newcommand{\R}{\mathbb{R}}

